\pdfoutput=1
\documentclass[letterpaper,11pt]{article}
\usepackage{amsmath,amsthm,amssymb}
\usepackage{graphicx, hyperref, fullpage, color}
\usepackage{algpseudocode, algorithm}
\usepackage{mathtools}
\usepackage{multirow}
\usepackage{array}

\newcolumntype{C}[1]{>{\centering\let\newline\\\arraybackslash\hspace{0pt}}m{#1}}

\usepackage[colorinlistoftodos,bordercolor=orange,backgroundcolor=orange!20,linecolor=orange,textsize=scriptsize]{todonotes}

\newcommand{\R}{\mathbb{R}}
\newcommand{\Exp}{\mathbf{E}}
\newcommand{\Prob}{\mathbf{P}}

\newcommand{\grad}{ g} 
\newcommand{\yy}{ y}
\newcommand{\xx}{ x}
\newcommand{\ff}{ f}
\newcommand{\kk}{ j}
\newcommand{\cc}{ c} 

\newcommand{\eqdef}{\overset{\text{def}}{=}}

\newcommand{\peter}[1]{{\color{red} #1}}

\newtheorem{theorem}{Theorem}
\newtheorem{lemma}[theorem]{Lemma}

\newtheorem{assumption}[theorem]{Assumption}

\title{Semi-Stochastic  Gradient Descent Methods}

\author{Jakub Kone\v{c}n\'{y} \footnote{School of Mathematics, The University of Edinburgh, United Kingdom (e-mail: J.Konecny@sms.ed.ac.uk)} \qquad \qquad  Peter Richt\'{a}rik \footnote{School of Mathematics, The University of Edinburgh, United Kingdom (e-mail: peter.richtarik@ed.ac.uk) \qquad
The work of both authors was supported by the Centre for Numerical Algorithms and Intelligent Software (funded by EPSRC grant EP/G036136/1 and the Scottish Funding Council). Both authors also thank the Simons Institute for the Theory of Computing, UC Berkeley, where this work was conceived and finalized. The work of P.R. was also supported by the EPSRC grant EP/I017127/1 (Mathematics for Vast Digital Resources) and EPSRC grant EP/K02325X/1 (Accelerated Coordinate Descent Methods for Big Data Problems). }\\\\ {\em School of Mathematics}\\
{\em University of Edinburgh}\\
{\em United Kingdom}}

\date{June 15, 2015 (first version: arXiv:1312.1666)}

\begin{document}

\maketitle

\begin{abstract} In this paper we  study the problem of minimizing the average of a large number ($n$) of smooth convex loss functions. We propose a new method, S2GD  (Semi-Stochastic Gradient Descent), which runs for one or several epochs in each of which a single full gradient and a random number of stochastic gradients is computed, following a geometric law. The total work needed for the method to output an $\varepsilon$-accurate solution in expectation, measured in the number of passes over data, or equivalently, in units equivalent to the computation of a single gradient of the empirical loss, is $O((n / \kappa)\log(1/\varepsilon))$, where $\kappa$ is the condition number. This is achieved by running the method for  $O(\log(1/\varepsilon))$ epochs,  with a single gradient evaluation and $O(\kappa)$ stochastic gradient evaluations in each. The SVRG method of Johnson and Zhang \cite{svrg} arises as a special case. If our method is limited to a single epoch only,  it needs to evaluate at most $O((\kappa/\varepsilon)\log(1/\varepsilon))$ stochastic gradients. In contrast, SVRG requires $O(\kappa/\varepsilon^2)$ stochastic gradients. To illustrate our theoretical results, S2GD only needs the workload equivalent to about 2.1 full gradient evaluations to find an $10^{-6}$-accurate solution for a problem with $n=10^9$ and $\kappa=10^3$.

%
\end{abstract}


\section{Introduction}

Many problems in data science (e.g., machine learning, optimization and statistics) can be cast as loss minimization problems of the form \begin{equation}\label{eq:main}\min_{x \in \mathbb{R}^d} \ff(x),\end{equation}
where
\begin{equation}\label{eq:main2} \ff(x) \eqdef \frac{1}{n} \sum_{i=1}^n \ff_i(x).\end{equation}

Here $d$ typically denotes the number of features / coordinates, $n$ the number of examples, and $\ff_i(x)$ is the loss incurred on example $i$. That is, we are seeking to find a predictor $x \in \mathbb{R}^d$ minimizing the average loss $\ff(x)$. In big data applications, $n$ is typically very large; in particular, $n \gg d$.  

Note that this formulation includes more typical formulation of $L2$-regularized objectives --- $\ff(x) = \frac{1}{n} \sum_{i=1}^n \tilde{\ff}_i(x) + \frac{\lambda}{2} \| x \|^2. $ We hide the regularizer into the function $\ff_i(x)$ for the sake of simplicity of resulting analysis.

\subsection{Motivation}

Let us now briefly review  two basic approaches to solving problem \eqref{eq:main}.
\begin{enumerate}
\item \emph{Gradient Descent.} Given $x_k \in \R^d$, the gradient descent (GD) method sets $$ x_{k+1} = x_k - h \ff'(x_k), $$ where $h$ is a stepsize parameter and $\ff'(x_k)$ is the gradient of $\ff$ at $x_k$. We will refer to $\ff'(x)$ by the name \emph{full gradient}. In order to compute $\ff'(x_k)$, we need to compute the gradients of $n$ functions. Since $n$ is big, it is prohibitive to do this at every iteration. 

\item \emph{Stochastic Gradient Descent (SGD).} Unlike gradient descent,  stochastic gradient descent \cite{nemirovski2009robust, tongSGD}  instead picks a random $i$ (uniformly) and updates $$ x_{k+1} = x_k - h \ff'_i(x_k). $$ Note that this strategy drastically reduces the amount of work that needs to be done in each iteration (by the factor of $n$). Since \[ \Exp(\ff_i'(x_k))  = \ff'(x_k),\] we have an unbiased estimator of the full gradient. Hence, the gradients of the component functions $\ff_1,\dots,\ff_n$ will be referred to as \emph{stochastic gradients}.  A practical issue with SGD is that consecutive stochastic gradients may vary a lot or even point in opposite directions. This slows down the performance of SGD. On balance, however, SGD is preferable to GD in applications where low accuracy solutions are sufficient. In such cases usually only a small number of passes through the data (i.e., work equivalent to a small number of full gradient evaluations) are needed to find an acceptable $x$. For this reason, SGD is extremely popular in fields such as machine learning.

\end{enumerate}

 In order to improve upon GD, one needs to reduce the cost of computing a gradient. In order to improve upon SGD, one has to reduce the variance of the stochastic gradients.  In this paper we propose and analyze a \emph{Semi-Stochastic Gradient Descent} (S2GD) method. Our  method combines GD and SGD steps and reaps the benefits of both algorithms: it inherits the stability and speed of GD and at the same time retains the work-efficiency of SGD.

\subsection{Brief literature review}

Several recent papers, e.g., Richt\'{a}rik \& Tak\'{a}\v{c} \cite{richtarik}, Le Roux, Schmidt \& Bach \cite{SAG,SAGjournal2013}, Shalev-Shwartz \& Zhang \cite{SDCA} and Johnson \& Zhang \cite{svrg}  proposed methods which achieve such a variance-reduction effect, directly or indirectly. These methods enjoy linear convergence rates when applied to minimizing  smooth  strongly convex loss functions. 

The method in \cite{richtarik} is known as Random Coordinate Descent for Composite functions (RCDC), and can be either applied directly to \eqref{eq:main}---in which case a single iteration requires $O(n)$ work for a dense problem, and $O(d \log(1/\varepsilon))$ iterations in total---or  to a dual version of \eqref{eq:main}, which requires $O(d)$ work per iteration and $O((n+\kappa)\log(1/\varepsilon))$ iterations in total. Application of a coordinate descent method to a dual formulation of \eqref{eq:main} is generally referred to as Stochastic Dual Coordinate Ascent (SDCA) \cite{SDCA-2008}. The algorithm in \cite{SDCA} 
exhibits this duality, and the method in \cite{minibatch-ICML2013} extends the primal-dual framework to the parallel /  mini-batch setting. Parallel  and distributed stochastic coordinate descent methods were studied in \cite{RT:PCDM, FR:SPCDM2013, RT:Hydra2013}.

Stochastic Average Gradient (SAG) \cite{SAG} is one of the first SGD-type methods, other than coordinate descent methods, which were shown to exhibit  linear convergence. The method of Johnson and Zhang \cite{svrg}, called Stochastic Variance Reduced Gradient (SVRG),  arises as a special case in our setting for a suboptimal choice of a single parameter of our method. The Epoch Mixed Gradient Descent (EMGD) method \cite{zhanglijun} is similar in spirit to SVRG, but achieves a quadratic dependence on the condition number instead of a linear dependence, as is the case with SAG, SVRG and with our method. 

For classical work on semi-stochastic gradient descent methods we refer\footnote{We thank Zaid Harchaoui who pointed us to these papers a few days before we posted our work to arXiv.} the reader to the papers of Murti and Fuchs \cite{MF79, MF86}.  

\subsection{Outline}

We start in Section~\ref{SEC:S2GD} by describing two algorithms: S2GD, which we analyze, and S2GD+, which we do not analyze, but which exhibits superior performance in practice. We then move to summarizing some of the main contributions  of this paper in Section~\ref{SEC:summary}. Section~\ref{SEC:strong} is devoted to establishing expectation and high probability complexity results for S2GD in the case of a strongly convex loss. The results are generic in that the parameters of the method are set arbitrarily. Hence, in Section~\ref{SEC:OPT} we study the problem of choosing the parameters optimally, with the goal of minimizing the total workload (\# of processed examples) sufficient to produce a result of sufficient accuracy. 
In Section~\ref{SEC:convex} we establish high probability complexity bounds for S2GD applied to a non-strongly convex loss function. Finally, in Section~\ref{SEC:NUMERICS} we perform very encouraging numerical experiments on real and artificial problem instances. A brief conclusion can be found in Section~\ref{SEC: CONCLUDE}.

\section{Semi-Stochastic Gradient Descent} \label{SEC:S2GD}

In this section we describe two novel algorithms: S2GD and S2GD+. We analyze the former only. The latter, however, has superior convergence properties in our experiments. 

We assume throughout the paper that the functions $\ff_i$ are convex and $L$-smooth.

\begin{assumption}\label{ass:Lip}
The functions $f_1,\dots,f_n$ have Lipschitz continuous gradients with constant $L > 0$ (in other words, they are $L$-smooth). That is, for all $x,z \in \R^d$ and all $i=1,2,\dots,n$,
$$ \ff_i(z) \leq \ff_i(x) + \langle \ff_i'(x), z - x \rangle + \frac{L}{2} \| z - x \|^2.$$
(This implies that the gradient of $\ff$ is Lipschitz with constant $L$, and hence $\ff$ satisfies the same inequality.)
\end{assumption}

In one part of the paper (Section~\ref{SEC:strong}) we also make the following additional assumption:

\begin{assumption}\label{ass:strong}
The average loss $f$ is $\mu$-strongly convex, $\mu>0$. That is, for all $x,z \in \R^d$,
\begin{equation}
f(z) \geq f(x) + \langle f'(x), z - x \rangle + \frac{\mu}{2} \| z - x \|^2. \label{SVRGstrcvx}
\end{equation}
(Note that, necessarily, $\mu\leq L $.)
\end{assumption}

\subsection{S2GD}

Algorithm~\ref{SVRG} (S2GD) depends on three parameters: stepsize $h$, constant $m$ limiting the number of stochastic gradients computed in a single epoch, and a $\nu \in [0,\mu]$, where $\mu$ is the strong convexity constant of $f$. In practice, $\nu$ would be a known lower bound on $\mu$. Note that the algorithm works also without any knowledge of the strong convexity parameter --- the case of $\nu = 0$.

\begin{algorithm}
\begin{algorithmic}
\State \textbf{parameters:} $m$ = max \# of stochastic steps per epoch, $h$ = stepsize, $\nu$ = lower bound on $\mu$
\For {$\kk = 0, 1, 2, \dots$}
	\State $\grad_{\kk} \gets \frac{1}{n} \sum_{i=1}^n f_i'(\xx_{\kk})$
	\State $\yy_{\kk,0} \gets \xx_{\kk}$
	\State Let $t_{\kk} \gets t$ with probability $(1 - \nu h)^{m-t} / \beta $ for $t = 1, 2, \dots, m$
	\For {$t = 0$ to $t_{\kk}-1$}
		\State Pick $i \in \{ 1, 2, \dots, n \}$, uniformly at random
		\State $ \yy_{\kk,t+1} \gets \yy_{\kk,t} - h \left( \grad_{\kk} + f_i'(\yy_{\kk,t}) - f_i'(\xx_{\kk})  \right) $
	\EndFor
	\State $\xx_{\kk+1} \gets \yy_{\kk, t_{\kk}}$
\EndFor
\end{algorithmic}

\caption{Semi-Stochastic Gradient Descent (S2GD)}
\label{SVRG}
\end{algorithm}

The method has an outer loop, indexed by epoch counter $\kk$, and an inner loop, indexed by $t$. In each epoch $\kk$, the method first computes $\grad_{\kk}$---the \emph{full} gradient of $f$ at $\xx_{\kk}$. Subsequently, the method produces a random number $t_{\kk} \in [1,m]$ of steps, following a geometric law, where
\begin{equation}\label{eq:beta} \beta \eqdef \sum_{t = 1}^m (1 - \nu h)^{m-t},\end{equation}
 with only \emph{two stochastic gradients} computed in each step\footnote{It is possible to get away with computinge only a \emph{single} stochastic gradient per inner iteration, namely $f_i'(\yy_{\kk,t})$, at the cost of having to store in memory $f'_i(\xx_\kk)$ for $i=1,2,\dots,n$. This, however, will be impractical for big $n$.}.  For each $t = 0, \dots, t_{\kk}-1$, the stochastic gradient $f'_i(\xx_{\kk})$ is subtracted from $\grad_{\kk}$, and $f'_i(\yy_{\kk,t-1})$ is added to $\grad_{\kk}$, which ensures that,  one has \[\Exp(\grad_{\kk} + f'_i(\yy_{\kk,t}) - f'_i(\xx_{\kk}) ) = f'(\yy_{\kk,t}),\]
 where the expectation is with respect to the random variable $i$.

Hence, the algorithm is   stochastic gradient descent -- albeit executed in a nonstandard way (compared to the traditional implementation described in the introduction).

Note that for all $\kk$, the expected number of iterations of the inner loop, $\Exp(t_{\kk})$, is equal to
\begin{equation} \label{eq:syhs7s5hs}  \xi = \xi(m,h) \eqdef  \sum_{t=1}^{m} t \frac{(1-\nu h)^{m-t}}{\beta}.\end{equation}
Also note that $\xi \in [\tfrac{m+1}{2},m)$, with the lower bound attained for $\nu=0$, and the upper bound for $\nu h \to 1$. 

\subsection{S2GD+}

We also implement Algorithm~\ref{alg:S2GD+}, which we call S2GD+. In our experiments, the performance of this method is superior to all methods we tested, including S2GD. However, we do not analyze the complexity of this method and leave this as an open problem.

\begin{algorithm}
\begin{algorithmic}
\State \textbf{parameters:} $\alpha \geq 1$ (e.g., $\alpha=1$)
\State 1. Run SGD for a single pass over the data (i.e., $n$ iterations); output $x$
\State 2. Starting from $x_0=x$, run a version of S2GD in which $t_j = \alpha n$ for all $j$
\end{algorithmic}
\caption{S2GD+}
\label{alg:S2GD+}
\end{algorithm}

In brief, S2GD+ starts by running SGD for 1 epoch (1 pass over the data) and then switches to a variant of S2GD in which the number of the inner iterations, $t_j$, is not random, but fixed to be $n$ or a  small multiple of $n$. 

The motivation for this method is the following. It is common knowledge that SGD is able to progress much more in one pass over the data than GD (where this would correspond to a single gradient step). However, the very first step of S2GD is the computation of the full gradient of $f$. Hence, by starting with a single pass over data using SGD and \emph{then} switching to S2GD, we obtain a superior method in practice.\footnote{Using a single pass of SGD as an initialization strategy was already considered in \cite{SAG}. However, the authors claim that their implementation of vanilla SAG did not benefit from it. S2GD does benefit from such an initialization due to it starting, in theory, with a (heavy) full gradient computation.}


\section{Summary of Results} \label{SEC:summary}

In this section we summarize some of the main results and contributions of this work.

\begin{enumerate}
\item \textbf{Complexity for strongly convex $f$.} If $f$ is strongly convex, S2GD needs \begin{equation}\label{eq:sjss5s4s}{\cal W} = O((n+\kappa)\log(1/\varepsilon))\end{equation} work (measured as the total number of evaluations of the stochastic gradient, accounting for the full gradient evaluations as well) to output an $\varepsilon$-approximate solution (in expectation or in high probability), where $\kappa=L/\mu$ is the condition number. This is achieved by running S2GD with stepsize $h=O(1/L)$, $\kk = O(\log (1/\varepsilon))$ epochs (this is also equal to the number of full gradient evaluations) and $m = O(\kappa )$ (this is also roughly equal to the number of stochastic gradient evaluations in a single epoch). The complexity results are stated in detail in Sections~\ref{SEC:strong} and \ref{SEC:OPT} (see Theorems~\ref{thm:expalpha}, \ref{thm:hpresult} and \ref{thm:main2}; see also \eqref{eq:m:nuismu0_2} and \eqref{eq:0sjsys8jns}).

\item \textbf{Comparison with existing results.} This complexity result \eqref{eq:sjss5s4s} matches the best-known results obtained for strongly convex losses in recent work such as \cite{SAG}, \cite{svrg} and \cite{zhanglijun}. Our treatment is most closely related to \cite{svrg}, and contains their method (SVRG) as a special case. However, our complexity results have better constants, which has  a discernable effect in practice. In Table~\ref{tbl:comparison} we compare our results in the strongly convex case with other existing results for different algorithms.

\begin{table}[h!]
\begin{center}
\begin{tabular}{|c|c|}
\hline
Algorithm & Complexity/Work \\
\hline \hline
Nesterov's algorithm & $O\left(\sqrt{\kappa}n\log(1/\varepsilon)\right)$ \\ EMGD & $O\left((n + \kappa^2)\log(1/\varepsilon)\right)$ \\ 
SAG & $O\left(n\log(1/\varepsilon)\right)$ \\ 
SDCA & $O\left((n + \kappa)\log(1/\varepsilon)\right)$ \\ 
SVRG & $O\left((n + \kappa)\log(1/\varepsilon)\right)$ \\ 
\hline
\textbf{S2GD} & $O\left((n + \kappa)\log(1/\varepsilon)\right)$ \\ \hline
\end{tabular}
\end{center}
\caption{Comparison of performance of selected methods suitable for solving \eqref{eq:main}. The complexity/work is measured in the number of stochastic gradient evaluations needed to find an $\varepsilon$-solution.}
\label{tbl:comparison}
\end{table}
We should note that the rate of convergence of Nesterov's algorithm \cite{nesterovIntro} is a deterministic result. EMGD and S2GD results hold with high probability. The remaining results hold in expectation. Complexity results for stochastic coordinate descent methods are also typically analyzed in the high probability regime \cite{richtarik}.

\item \textbf{Complexity for convex $f$.}
If $f$ is \emph{not} strongly convex, then we propose that S2GD be applied to a perturbed version of the problem, with strong convexity  constant $\mu=O(L/\varepsilon)$. An $\varepsilon$-accurate solution of the original problem is recovered with arbitrarily high probability (see Theorem~\ref{thm:hpresult2} in Section~\ref{SEC:convex}). The total work in this case is \[{\cal W}=O\left( \left(n+ L/\varepsilon)\right)\log\left(1/\varepsilon\right)\right),\]
that is, $\tilde{O}(1/\epsilon)$, which is better than the standard rate of SGD.


\item \textbf{Optimal parameters.} We derive formulas for optimal parameters of the method which (approximately) minimize the total workload, measured in the number of stochastic gradients computed (counting a single full gradient evaluation as $n$ evaluations of the stochastic gradient). In particular, we show that the method should be run for $O(\log(1/\varepsilon))$ epochs, with stepsize $h=O(1/L)$ and $m=O(\kappa)$. No such results were derived for SVRG in \cite{svrg}.

\item \textbf{One epoch.} In the case when S2GD is run for 1 epoch only, effectively limiting the number of full gradient evaluations to 1, we show that S2GD with $\nu=\mu$ needs  \[O(n+(\kappa/\varepsilon)\log(1/\varepsilon))\] work only (see Table~\ref{tbl:ssus8778}). This compares favorably with the optimal complexity in the $\nu = 0$ case (which reduces to SVRG), where the work needed is \[O(n+\kappa/\varepsilon^2).\] 

For two epochs one could just say that we need $\sqrt{\varepsilon}$ decrease in iach epoch, thus having complexity of $O(n+(\kappa/\sqrt{\varepsilon})\log(1/\sqrt{\varepsilon}))$. This is already better than general rate of SGD $(O(1 / \varepsilon)).$

\begin{table}[!h]
\begin{center}
\begin{tabular}{|c|l|c|}
\hline
\textbf{Parameters} & \textbf{Method} & \textbf{Complexity}\\
\hline
\phantom{-} & & \\
\begin{tabular}{c}
$\nu=\mu$, $\kk=O(\log(\tfrac{1}{\varepsilon}))$\\ \& $m=O(\kappa)$
\end{tabular}
  & Optimal S2GD & $O((n+\kappa)\log(\tfrac{1}{\varepsilon}))$\\ 
\phantom{-} & & \\
\hline
$m=1$ & GD & ---\\
$\nu=0$ & SVRG \cite{svrg} & $O((n+\kappa)\log(\tfrac{1}{\varepsilon}))$ \\
$\nu=0$, $\kk=1$, $m=O(\tfrac{\kappa}{\varepsilon^2})$ & Optimal SVRG with 1 epoch & $O(n+\tfrac{\kappa}{\varepsilon^2})$\\
$\nu=\mu$, $\kk=1$, $m = O(\tfrac{\kappa}{\varepsilon} \log(\tfrac{1}{\varepsilon}))$ & Optimal S2GD with 1 epoch & $O(n+\tfrac{\kappa}{\varepsilon}\log(\tfrac{1}{\varepsilon}))$\\
\hline
\end{tabular}
\end{center}
\caption{Summary of complexity results and special cases. Condition number: $\kappa = L/\mu$ if $f$ is $\mu$-strongly convex and $\kappa=2L/\varepsilon$ if $f$ is \emph{not} strongly convex and $\epsilon \leq L$. }
\label{tbl:ssus8778}
\end{table}

\item \textbf{Special cases.} 
GD and SVRG arise as special cases of S2GD, for $m=1$ and $\nu=0$, respectively.\footnote{While S2GD reduces to GD for $m=1$, our \emph{analysis} does not say anything meaningful in the $m=1$ case - it is too coarse to cover this case. This is also the reason behind the empty space in the ``Complexity'' box column for GD in Table~\ref{tbl:ssus8778}.} 

\item \textbf{Low memory requirements.} Note that SDCA and SAG, unlike SVRG and S2GD, need to store  all gradients $f'_i$ (or dual variables) throughout the iterative process. While this may not be a problem for a modest sized optimization task,  this requirement makes such methods less suitable for problems with very large $n$.

\item \textbf{S2GD+.} We propose a ``boosted'' version of S2GD, called S2GD+, which we do not analyze. In our experiments, however, it performs vastly superior to all other methods we tested, including GD, SGD, SAG and S2GD. S2GD alone is better than both GD and SGD if a highly accurate solution is required. The performance of S2GD and SAG is roughly comparable, even though in our experiments S2GD turned to have an edge.

\end{enumerate}


\section{Complexity Analysis: Strongly Convex Loss}\label{SEC:strong}

For the purpose of the analysis, let 
\begin{equation}
\mathcal{F}_{\kk,t} \eqdef \sigma( \xx_1, \xx_2, \dots, \xx_{\kk}; \yy_{\kk,1}, \yy_{\kk,2}, \dots, \yy_{\kk,t} )
\label{sigmaalgebra}
\end{equation}
be the $\sigma$-algebra generated by the relevant history of S2GD. We first isolate an auxiliary result.

\begin{lemma}
Consider the S2GD algorithm. For any fixed epoch number $\kk$, the following identity holds:
\begin{equation}  \Exp(f(\xx_{\kk+1})) = \frac{1}{\beta}\sum_{t=1}^m  (1 - \nu h)^{m-t}\Exp\left(f(\yy_{\kk,t-1}) \right).\label{uberLemma}\end{equation}
\end{lemma}

\begin{proof} By the tower property of expectations and the definition of $\xx_{\kk+1}$ in the algorithm, we obtain
\begin{eqnarray*}
\Exp( f(\xx_{\kk+1}) ) \;\;=\;\; \Exp \left( \Exp( f(\xx_{\kk+1})\; |\; \mathcal{F}_{\kk, m}) \right) &=& \Exp \left( \sum_{t=1}^m \frac{(1 - \nu h)^{m-t}}{\beta} f(\yy_{\kk,t-1}) \right)\\
&=&\frac{1}{\beta}\sum_{t=1}^m  (1 - \nu h)^{m-t}\Exp\left(f(\yy_{\kk,t-1}) \right).
\end{eqnarray*}
\end{proof}

We now state and prove the main result of this section. 


\begin{theorem} \label{thm:MAIN} Let Assumptions \ref{ass:Lip} and \ref{ass:strong} be satisfied. Consider the S2GD algorithm applied to solving problem \eqref{eq:main}. Choose $0\leq \nu \leq \mu$, $0< h  < \frac{1}{2L}$, and let $m$ be sufficiently large so that \begin{equation}\label{eq:hshshs7} \cc \eqdef \frac{(1 -\nu h)^m}{\beta \mu h (1 - 2Lh)} + \frac{2(L - \mu)h}{1 - 2Lh} < 1. \end{equation}
Then we have the following convergence in expectation:
\begin{equation}\label{eq:s8shs7} \Exp\left( f(\xx_{\kk}) - f(\xx_*) \right) \leq \cc^{\kk} (f(\xx_0) - f(\xx_*)). \end{equation}
\label{thm:expalpha}
\end{theorem}

Before we proceed to proving the theorem, note that in the special case with $\nu = 0$, we recover the result of Johnson and Zhang \cite{svrg} (with a minor improvement in the second term of $\cc$ where $L$ is replaced by $L-\mu$), namely \begin{equation}
\cc = \frac{1}{\mu h (1 - 2Lh) m} + \frac{2(L - \mu)h}{1 - 2Lh}. \label{eq:nuiszero}
\end{equation} If we set $\nu = \mu$, then $\cc$ can be written in the form (see \eqref{eq:beta})
\begin{equation}
\cc = \frac{(1 - \mu h)^m}{(1 - (1 - \mu h)^m) (1 - 2Lh)} + \frac{2(L - \mu)h}{1 - 2Lh}.
\label{eq:nuismu}
\end{equation}

Clearly, the latter $c$ is a major improvement on the former one. We shall elaborate on this further later.

\begin{proof}
It is well-known \cite[Theorem 2.1.5]{nesterovIntro} that since the functions $f_i$ are $L$-smooth, they necessarily satisfy the following inequality:
$$ \| f_i'(x) - f_i'(\xx_*) \|^2 \leq 2L \left[ f_i(x) - f_i(\xx_*) - \langle f_i'(\xx_*), x - \xx_* \rangle \right] .$$
By summing these inequalities for  $i = 1, \dots, n$, and using $f'(\xx_*) = 0,$ we get
\begin{equation}
\frac{1}{n} \sum_{i=1}^n \| f_i'(x) - f_i'(\xx_*) \|^2 \leq 2L \left[ f(x) - f(\xx_*) - \langle f'(\xx_*), x - \xx_* \rangle \right] = 2L (f(x) - f(\xx_*)).
\label{SVRGbound}
\end{equation}

Let $G_{\kk, t} \eqdef \grad_{\kk} +f_i'(\yy_{\kk,t-1}) - f'_i(\xx_{\kk}) $ be the direction of update at ${\kk}^{th}$ iteration in the outer loop and $t^{th}$ iteration in the inner loop. Taking expectation with respect to $i$, conditioned on the $\sigma$-algebra $\mathcal{F}_{\kk, t-1}$ \eqref{sigmaalgebra}, we obtain\footnote{For simplicity, we supress the $\Exp(\cdot \;|\; \mathcal{F}_{\kk,t-1})$ notation here.}
\begin{eqnarray}
\notag
\Exp \left( \|G_{\kk,t} \|^2  \right) &=& \Exp \left( \| f_i'(\yy_{\kk,t-1}) - f_i'(\xx_*) - f_i'(\xx_{\kk}) + f_i'(\xx_*) + \grad_{\kk} \|^2  \right) \\ \notag
&\leq&  2\Exp\left(\| f_i'(\yy_{\kk,t-1}) - f_i'(\xx_*) \|^2\right)  + 2 \Exp\left(\| \left[ f_i'(\xx_{\kk}) - f_i'(\xx_*) \right] - f'(\xx_{\kk}) \|^2 \right) \\ \notag
&=& 2\Exp\left(\| f_i'(\yy_{\kk,t-1}) - f_i'(\xx_*) \|^2 \right) \\ \notag
&& \qquad + 2 \Exp\left(\| f_i'(\xx_{\kk}) - f_i'(\xx_*) \|^2 \right) - 4\Exp \left(\left\langle f'(\xx_{\kk}), f_i'(\xx_{\kk}) - f_i'(\xx_*) \right\rangle \right)+ 2\| f'(\xx_{\kk}) \|^2 \\ \notag
&\overset{\eqref{SVRGbound}} {\leq}& 4L \left[ f(\yy_{\kk,t-1}) - f(\xx_*) + f(\xx_{\kk}) - f(\xx_*) \right] - 2\| f'(\xx_{\kk}) \|^2 - 4\langle f'(\xx_{\kk}), f'(\xx_*) \rangle\\
& \overset{\eqref{SVRGstrcvx}} {\leq} & 4L \left[ f(\yy_{\kk,t-1}) - f(\xx_*) \right] + 4(L-\mu) \left[ f(\xx_{\kk}) - f(\xx_*) \right].
\label{expvstvariance}
\end{eqnarray}
Above we have used the bound $\| x'+x'' \|^2 \leq 2\|x'\|^2 + 2\|x''\|^2$ and the fact that
\begin{equation}
\Exp( G_{\kk, t} \;|\; \mathcal{F}_{\kk,t-1}) = f'(\yy_{\kk, t-1}).
\label{expvst}
\end{equation}


We now study the expected distance to the optimal solution (a standard approach in the analysis of gradient methods):
\begin{eqnarray}
\Exp (\|\yy_{\kk,t} - \xx_* \|^2 \;|\; \mathcal{F}_{\kk, t-1} ) &=& \|\yy_{\kk,t-1} - \xx_* \|^2 - 2h \langle \Exp( G_{\kk,t} \;|\; \mathcal{F}_{\kk, t-1} ), \yy_{\kk, t-1} - \xx_* \rangle \notag \\
 && \qquad + h^2 \Exp (\|G_{\kk,t}\|^2 \;|\; \mathcal{F}_{\kk, t-1} ) \notag \\
& \overset{\eqref{expvstvariance}+ \eqref{expvst}} {\leq} & \|\yy_{\kk,t-1} - \xx_* \|^2 - 2h \langle f'(\yy_{\kk,t-1}), \yy_{\kk, t-1} - \xx_* \rangle \notag \\
& & \qquad + 4Lh^2 \left[ f(\yy_{\kk,t-1}) - f(\xx_*) \right] + 4(L - \mu)h^2\left[ f(\xx_{\kk}) - f(\xx_*) \right] \notag\\
& \overset{\eqref{SVRGstrcvx}}{\leq} & \|\yy_{\kk,t-1} - \xx_* \|^2 - 2h \left[ f(\yy_{\kk,t-1}) - f(\xx_*) \right] - \nu h \|\yy_{\kk,t-1} - \xx_* \|^2 \notag\\
& & \qquad + 4Lh^2 \left[ f(\yy_{\kk,t-1}) - f(\xx_*) \right] + 4(L - \mu)h^2\left[ f(\xx_{\kk}) - f(\xx_*) \right] \notag\\
&=& (1 - \nu h) \|\yy_{\kk,t-1} - \xx_* \|^2 - 2h(1 - 2Lh)[f(\yy_{\kk,t-1}) - f(\xx_*)] \notag\\ 
& & \qquad + 4(L - \mu)h^2[f(\xx_{\kk}) - f(\xx_*)]. \label{distbound}
\end{eqnarray}

By rearranging the terms in \eqref{distbound} and taking expectation over the $\sigma$-algebra $\mathcal{F}_{\kk, t-1}$, we get the following inequality:

\begin{align}
\Exp(\|\yy_{\kk,t} - \xx_* \|^2) + 2h(1 - 2Lh)&\Exp(f(\yy_{\kk,t-1}) - f(\xx_*)) \notag\\ 
&\leq (1 - \nu h) \Exp(\|\yy_{\kk,t-1} - \xx_* \|^2) + 4(L - \mu)h^2\Exp(f(\xx_{\kk}) - f(\xx_*)).\label{eq:shs7shs}
\end{align}

Finally, we can analyze what happens after one iteration of the outer loop of S2GD, i.e., between two computations of the full gradient. By summing up inequalities \eqref{eq:shs7shs} for $t = 1, \dots, m$, with inequality $t$ multiplied by $ (1 - \nu h)^{m-t}$, we get the left-hand side
\begin{eqnarray}
LHS &=& \Exp(\|\yy_{\kk,m} - \xx_*\|^2) + 2h(1 - 2Lh) \sum_{t = 1}^m (1 - \nu h)^{m-t} \Exp(f(\yy_{\kk,t-1}) - f(\xx_*))\notag \\
&\overset{\eqref{uberLemma}}{=} & \Exp(\|\yy_{\kk,m} - \xx_*\|^2) + 2\beta h(1 - 2Lh) \Exp (f(\xx_{\kk+1}) - f(\xx_*)),\notag 
\end{eqnarray}
and the right-hand side
\begin{eqnarray}
RHS &=& (1 - \nu h)^m \Exp(\| \xx_{\kk} - \xx_* \|^2) + 4\beta(L - \mu)h^2 \Exp(f(\xx_{\kk}) - f(\xx_*)) \notag\\
&\overset{\eqref{SVRGstrcvx}}{\leq} &\frac{2(1 - \nu h)^m}{\mu} \Exp(f(\xx_{\kk}) - f(\xx_*)) + 4\beta(L - \mu)h^2 \Exp(f(\xx_{\kk}) - f(\xx_*))\notag \\
&= & 2\left(\frac{(1 - \nu h)^m }{\mu} + 2\beta(L - \mu)h^2\right) \Exp(f(\xx_{\kk}) - f(\xx_*)).\notag 
\end{eqnarray}
Since $LHS \leq RHS$, we finally conclude with 
\begin{eqnarray*}
\Exp(f(\xx_{\kk+1}) - f(\xx_*)) &\leq & \cc \Exp(f(\xx_{\kk}) - f(\xx_*)) - \frac{\Exp(\| \yy_{\kk,m} - \xx_* \|^2)}{2\beta h (1 - 2Lh)} \;\;\leq \;\; \cc \Exp(f(\xx_{\kk}) - f(\xx_*)).
\end{eqnarray*}
\end{proof}

Since we have established linear convergence of expected values, a high probability result can be obtained in a straightforward way using Markov inequality. 

\begin{theorem}
Consider the setting of Theorem~\ref{thm:MAIN}. Then, for any $ 0 < \rho < 1 $, $ 0 < \varepsilon < 1 $ and  \begin{equation}
\kk \geq \frac{\log \left( \frac{1}{\varepsilon \rho}\right)}{\log \left(\frac{1}{\cc}\right)}, 
\label{eq:hprobs}
\end{equation} 
we have \begin{equation}\label{eq:sjnd8djd} \Prob\left(\frac{f(\xx_{\kk}) - f(\xx_*)}{f(\xx_0)-f(\xx_*)} \leq \varepsilon \right) \geq 1 - \rho. \end{equation}
\label{thm:hpresult}
\end{theorem}

\begin{proof}
This follows directly from  Markov inequality and Theorem~\ref{thm:expalpha}:
$$ \Prob(f(\xx_{\kk}) - f(\xx_*) > \varepsilon (f(\xx_0)-f(\xx_*)) \overset{\eqref{eq:s8shs7}}{\leq} \frac{\Exp(f(\xx_{\kk}) - f(\xx_*))}{\varepsilon(f(\xx_0)-f(\xx_*))} \leq \frac{\cc^\kk}{\varepsilon} \overset{\eqref{eq:hprobs}}{\leq} \rho $$
\end{proof}

This result will be also useful when treating the non-strongly convex case.

\section{Optimal Choice of Parameters}\label{SEC:OPT}

The goal of this section is to provide insight into the choice of parameters of  S2GD; that is, the number of epochs (equivalently, full gradient evaluations) $\kk$, the maximal number of steps in each epoch $m$, and the stepsize $h$. The remaining parameters ($L, \mu, n$) are inherent in the problem and we will hence treat them in this section as given.

In particular, ideally we wish to find parameters $\kk$, $m$ and $h$  solving the following optimization problem:
\begin{equation}\min_{\kk,m,h}  \quad \tilde{{\cal W}}(\kk,m,h) \eqdef \kk(n+2\xi(m,h))
\label{eq:pracReq},\end{equation}
subject to \begin{equation}\label{eq:sjs8s} \Exp(f(\xx_{\kk})-f(\xx_*))\leq \varepsilon(f(\xx_0)-f(\xx_*)).\end{equation}
 Note that $\tilde{{\cal W}}(\kk,m,h)$ is the \emph{expected work}, measured by the number  number of stochastic gradient evaluations, performed by  S2GD when running for $\kk$ epochs. Indeed, the   evaluation of $\grad_{\kk}$  is equivalent to $n$ stochastic gradient evaluations, and each epoch further computes on average $2\xi(m,h)$ stochastic gradients (see \eqref{eq:syhs7s5hs}). Since $\tfrac{m+1}{2}\leq \xi(m,h) < m$, we can simplify and solve the problem with $\xi$ set to the conservative upper estimate $\xi=m$.

In view of \eqref{eq:s8shs7},  accuracy constraint \eqref{eq:sjs8s} is satisfied if $\cc$ (which depends on $h$ and $m$) and $\kk$ satisfy
\begin{equation}
\cc^{\kk} \leq \varepsilon.
\label{eq:pracReq2}
\end{equation}

We therefore instead consider the parameter fine-tuning problem

\begin{equation}\label{eq:shd6dhd7}\min_{\kk,m,h} {\cal W}(\kk,m,h) \eqdef \kk(n+2m) \qquad \text{subject to} \qquad \cc \leq \varepsilon^{1/{\kk}}.\end{equation}

In the following we (approximately) solve this problem in two steps. First, we fix $\kk$ and find (nearly) optimal $h=h(\kk)$ and $m=m(\kk)$. The problem  reduces to minimizing $m$ subject to $c \leq \varepsilon^{1/{\kk}}$ by fine-tuning $h$. While in the $\nu=0$ case it is possible to obtain closed form solution, this is not possible for $\nu>\mu$. 

However, it is still possible to obtain a good formula for $h(\kk)$ leading to expression for good $m(\kk)$ which depends on $\varepsilon$ in the correct way. We then plug the formula for $m(\kk)$ obtained this way  back into \eqref{eq:shd6dhd7}, and study the quantity ${\cal W}(\kk,m(\kk),h(\kk)) = \kk(n+2m(\kk))$ as a function of $\kk$, over which we optimize optimize at the end.


\begin{theorem}[Choice of parameters]\label{thm:main2}
Fix the number of epochs $\kk \geq 1$, error tolerance $0 < \varepsilon < 1$, and  let $\Delta = \varepsilon^{1/\kk}$. If we run S2GD  with  the stepsize
\begin{equation} \label{eq:shhsdd998}  
h = h(\kk) \eqdef \frac{1}{\frac{4}{\Delta}(L-\mu) + 2L}
\end{equation} and 
\begin{equation}
m \geq m(\kk) \eqdef
\begin{cases}
\left(\frac{4(\kappa-1)}{\Delta}+2 \kappa\right) \log\left(\tfrac{2}{\Delta} + \frac{2\kappa - 1}{\kappa-1}\right), & \quad \text{if} \quad \nu=\mu,\\
\frac{8(\kappa-1)}{\Delta^2} + \frac{8\kappa}{\Delta} + \frac{2\kappa^2}{\kappa-1},& \quad \text{if} \quad \nu=0,
\end{cases}
\label{eq:m:nuismu0}
\end{equation}
then  $\Exp(f(\xx_{\kk}) - f(\xx_*)) \leq \varepsilon (f(x_0)-f(x_*)).$

In particular, if we choose ${\kk}^* = \lceil \log (1/\varepsilon)\rceil$, then 
$\frac{1}{\Delta} \leq \exp(1)$, and
hence $m({\kk}^*) = O(\kappa)$, leading to the workload 
\begin{equation}{\cal W}({\kk}^*,m({\kk}^*),h({\kk}^*)) = \lceil \log\left(\tfrac{1}{\varepsilon}\right)\rceil (n+ O(\kappa)) = O\left((n+\kappa) \log\left(\tfrac{1}{\varepsilon}\right)\right).
\label{eq:0sjsys8jns}
\end{equation}

\end{theorem}

\begin{proof} We only need to show that $\cc \leq \Delta$, where
$\cc$ is given by \eqref{eq:nuismu} for $\nu=\mu$ and by  \eqref{eq:nuiszero} for $\nu=0$. We denote the two summands in expressions for $\cc$ as $\cc_1$ and $\cc_2$. We choose the $h$ and $m$ so that both $\cc_1$ and $\cc_2$ are smaller than $\Delta / 2$, resulting in $\cc_1 + \cc_2 = c \leq \Delta$. 

The stepsize $h$ is chosen so that
$$\cc_2 \eqdef  \frac{2(L - \mu)h}{1 - 2Lh} = \frac{\Delta}{2},$$
and hence it only remains to verify that $\cc_1 = \cc-\cc_2 \leq \frac{\Delta}{2}$. In the $\nu=0$ case, $m(\kk)$ is chosen so that $\cc-\cc_2=\frac{\Delta}{2}$. In the $\nu=\mu$ case, $\cc-\cc_2=\frac{\Delta}{2}$  holds for
$m = \log\left(\frac{2}{\Delta}+ \frac{2\kappa-1}{\kappa-1}\right)/\log\left(\frac{1}{1-H}\right)$, where $H = \left(\frac{4(\kappa-1)}{\Delta} + 2\kappa \right)^{-1}$. We only need to observe that $\cc$ decreases as $m$ increases, and apply the inequality $\log\left(\frac{1}{1-H}\right) \geq H$.

\end{proof}

We now comment on the above result:

\begin{enumerate}

\item \textbf{Workload.} Notice that for the choice of parameters $\kk^*$, $h=h(\kk^*)$, $m=m(\kk^*)$ and any $\nu \in [0,\mu]$, the method needs $\log(1/\varepsilon)$ computations of the full gradient (note this is independent of $\kappa$), and $O(\kappa \log(1/\varepsilon))$ computations of the stochastic gradient. This result, and special cases thereof, are summarized in Table~\ref{tbl:ssus8778}.

\item \textbf{Simpler formulas for $m$.} If $\kappa\geq 2$, we can instead of \eqref{eq:m:nuismu0} use the following (slightly worse but) simpler  expressions for $m(\kk)$, obtained from \eqref{eq:m:nuismu0} by using the bounds $1\leq\kappa-1$,  $\kappa-1\leq \kappa$ and $\Delta<1$ in appropriate places (e.g., $\tfrac{8\kappa}{\Delta} <\tfrac{8\kappa}{\Delta^2}$, $\tfrac{\kappa}{\kappa-1}\leq 2 < \tfrac{2}{\Delta^2}$):
\begin{equation}
m \geq \tilde{m}(\kk) \eqdef
\begin{cases}
\frac{6\kappa}{\Delta} \log\left(\tfrac{5}{\Delta}\right), & \quad \text{if} \quad \nu=\mu,\\
\frac{20\kappa}{\Delta^2},& \quad \text{if} \quad \nu=0.
\end{cases}
\label{eq:m:nuismu0_2}
\end{equation}

\item \textbf{Optimal stepsize in the $\nu=0$ case.} 
Theorem~\ref{thm:main2} does not claim to have solved problem \eqref{eq:shd6dhd7}; the problem in general does not have a closed form solution. However, in the $\nu=0$ case a closed-form formula can easily be obtained:
\begin{equation}\label{eq:sjs8djd} h(\kk) =  \frac{1}{\tfrac{4}{\Delta}(L-\mu)+ 4L}, \qquad \qquad m \geq m(\kk) \eqdef  \frac{8(\kappa-1)}{\Delta^2} + \frac{8\kappa}{\Delta}.\end{equation}
Indeed, for fixed $\kk$,  \eqref{eq:shd6dhd7} is equivalent to finding $h$ that minimizes $m$ subject to the constraint $\cc\leq \Delta$. In view of \eqref{eq:nuiszero}, this is equivalent to searching for $h>0$ maximizing the quadratic $h \to h(\Delta-2(\Delta L + L -\mu)h)$, which leads to \eqref{eq:sjs8djd}.

Note that both the stepsize $h(\kk)$ and the resulting $m(\kk)$ are slightly larger in Theorem~\ref{thm:main2} than in \eqref{eq:sjs8djd}. This is because in the theorem the stepsize was for simplicity chosen to satisfy $\cc_2=\frac{\Delta}{2}$, and hence is  (slightly) suboptimal. Nevertheless, the dependence of $m(\kk)$ on $\Delta$ is of the correct (optimal) order  in both cases. That is, $m(\kk) = O\left(\tfrac{\kappa}{\Delta}\log(\tfrac{1}{\Delta})\right)$ for $\nu=\mu$ and $m(\kk)=O\left(\tfrac{\kappa}{\Delta^2}\right)$    for $\nu=0$.

\item \textbf{Stepsize choice.} In cases when one does not have a good estimate of the strong convexity constant $\mu$ to determine the stepsize via \eqref{eq:shhsdd998}, one may choose suboptimal stepsize that does not depend on $\mu$ and derive similar results to those above. For instance, one may choose $h=\frac{\Delta}{6L}$.


\end{enumerate}

In Table~\ref{tbl:suboptbounds} we provide comparison of work needed for small values of $\kk$, and different values of $\kappa$ and $\varepsilon.$ Note, for instance, that for any problem with $n=10^9$ and $\kappa=10^3$, S2GD outputs a highly accurate solution ($\varepsilon=10^{-6}$) in the amount of work equivalent to $2.12$ evaluations of the full gradient of $f$!

\begin{table}[!h]
\begin{center}
\begin{tabular}{C{15pt}|c|c|}
\cline{2-3}
 & \multicolumn{2}{c|}{$\varepsilon = 10^{-3}, \kappa = 10^3$} \\ \hline
\multicolumn{1}{|c|}{$\kk$} & $\mathcal{W}_\mu(\kk)$ & $\mathcal{W}_0(\kk)$ \\ \hline \hline
\multicolumn{1}{|c|}{$1$} & $\textbf{1.06n}$ & $17.0n$ \\ 
\multicolumn{1}{|c|}{$2$} & $2.00n$ & $\textbf{2.03n}$ \\ 
\multicolumn{1}{|c|}{$3$} & $3.00n$ & $3.00n$ \\ 
\multicolumn{1}{|c|}{$4$} & $4.00n$ & $4.00n$ \\ 
\multicolumn{1}{|c|}{$5$} & $5.00n$ & $5.00n$ \\ 
\hline \end{tabular}
\quad
\begin{tabular}{C{15pt}|c|c|}
\cline{2-3}
 & \multicolumn{2}{c|}{$\varepsilon = 10^{-6}, \kappa = 10^3$} \\ \hline
\multicolumn{1}{|c|}{$\kk$} & $\mathcal{W}_\mu(\kk)$ & $\mathcal{W}_0(\kk)$ \\ \hline \hline
\multicolumn{1}{|c|}{$1$} & $116n$ & $10^7n$ \\ 
\multicolumn{1}{|c|}{$2$} & $\textbf{2.12n}$ & $34.0n$ \\ 
\multicolumn{1}{|c|}{$3$} & $3.01n$ & $\textbf{3.48n}$ \\ 
\multicolumn{1}{|c|}{$4$} & $4.00n$ & $4.06n$ \\ 
\multicolumn{1}{|c|}{$5$} & $5.00n$ & $5.02n$ \\ 
\hline \end{tabular}
\quad
\begin{tabular}{C{15pt}|c|c|}
\cline{2-3}
 & \multicolumn{2}{c|}{$\varepsilon = 10^{-9}, \kappa = 10^3$} \\ \hline
\multicolumn{1}{|c|}{$\kk$} & $\mathcal{W}_\mu(\kk)$ & $\mathcal{W}_0(\kk)$ \\ \hline \hline
\multicolumn{1}{|c|}{$2$} & $7.58n$ & $10^4n$ \\ 
\multicolumn{1}{|c|}{$3$} & $\textbf{3.18n}$ & $51.0n$ \\ 
\multicolumn{1}{|c|}{$4$} & $4.03n$ & $6.03n$ \\ 
\multicolumn{1}{|c|}{$5$} & $5.01n$ & $\textbf{5.32n}$ \\ 
\multicolumn{1}{|c|}{$6$} & $6.00n$ & $6.09n$ \\ 
\hline \end{tabular}
\quad
\newline
\newline

\begin{tabular}{C{15pt}|c|c|}
\cline{2-3}
 & \multicolumn{2}{c|}{$\varepsilon = 10^{-3}, \kappa = 10^6$} \\ \hline
\multicolumn{1}{|c|}{$\kk$} & $\mathcal{W}_\mu(\kk)$ & $\mathcal{W}_0(\kk)$ \\ \hline \hline
\multicolumn{1}{|c|}{$2$} & $4.14n$ & $35.0n$ \\ 
\multicolumn{1}{|c|}{$3$} & $\textbf{3.77n}$ & $8.29n$ \\ 
\multicolumn{1}{|c|}{$4$} & $4.50n$ & $\textbf{6.39n}$ \\ 
\multicolumn{1}{|c|}{$5$} & $5.41n$ & $6.60n$ \\ 
\multicolumn{1}{|c|}{$6$} & $6.37n$ & $7.28n$ \\ 
\hline \end{tabular}
\quad
\begin{tabular}{C{15pt}|c|c|}
\cline{2-3}
 & \multicolumn{2}{c|}{$\varepsilon = 10^{-6}, \kappa = 10^6$} \\ \hline
\multicolumn{1}{|c|}{$\kk$} & $\mathcal{W}_\mu(\kk)$ & $\mathcal{W}_0(\kk)$ \\ \hline \hline
\multicolumn{1}{|c|}{$4$} & $8.29n$ & $70.0n$ \\ 
\multicolumn{1}{|c|}{$5$} & $\textbf{7.30n}$ & $26.3n$ \\ 
\multicolumn{1}{|c|}{$6$} & $7.55n$ & $16.5n$ \\ 
\multicolumn{1}{|c|}{$8$} & $9.01n$ & $\textbf{12.7n}$ \\ 
\multicolumn{1}{|c|}{$10$} & $10.8n$ & $13.2n$ \\ 
\hline \end{tabular}
\quad
\begin{tabular}{C{15pt}|c|c|}
\cline{2-3}
 & \multicolumn{2}{c|}{$\varepsilon = 10^{-9}, \kappa = 10^6$} \\ \hline
\multicolumn{1}{|c|}{$\kk$} & $\mathcal{W}_\mu(\kk)$ & $\mathcal{W}_0(\kk)$ \\ \hline \hline
\multicolumn{1}{|c|}{$5$} & $17.3n$ & $328n$ \\ 
\multicolumn{1}{|c|}{$8$} & $\textbf{10.9n}$ & $32.5n$ \\ 
\multicolumn{1}{|c|}{$10$} & $11.9n$ & $21.4n$ \\ 
\multicolumn{1}{|c|}{$13$} & $14.3n$ & $\textbf{19.1n}$ \\ 
\multicolumn{1}{|c|}{$20$} & $21.0n$ & $23.5n$ \\ 
\hline \end{tabular}
\quad
\newline
\newline

\begin{tabular}{C{15pt}|c|c|}
\cline{2-3}
 & \multicolumn{2}{c|}{$\varepsilon = 10^{-3}, \kappa = 10^9$} \\ \hline
\multicolumn{1}{|c|}{$\kk$} & $\mathcal{W}_\mu(\kk)$ & $\mathcal{W}_0(\kk)$ \\ \hline \hline
\multicolumn{1}{|c|}{$6$} & $378n$ & $1293n$ \\ 
\multicolumn{1}{|c|}{$8$} & $\textbf{358n}$ & $1063n$ \\ 
\multicolumn{1}{|c|}{$11$} & $376n$ & $\textbf{1002n}$ \\ 
\multicolumn{1}{|c|}{$15$} & $426n$ & $1058n$ \\ 
\multicolumn{1}{|c|}{$20$} & $501n$ & $1190n$ \\ 
\hline \end{tabular}
\quad
\begin{tabular}{C{15pt}|c|c|}
\cline{2-3}
 & \multicolumn{2}{c|}{$\varepsilon = 10^{-6}, \kappa = 10^9$} \\ \hline
\multicolumn{1}{|c|}{$\kk$} & $\mathcal{W}_\mu(\kk)$ & $\mathcal{W}_0(\kk)$ \\ \hline \hline
\multicolumn{1}{|c|}{$13$} & $737n$ & $2409n$ \\ 
\multicolumn{1}{|c|}{$16$} & $\textbf{717n}$ & $2126n$ \\ 
\multicolumn{1}{|c|}{$19$} & $727n$ & $2025n$ \\ 
\multicolumn{1}{|c|}{$22$} & $752n$ & $\textbf{2005n}$ \\ 
\multicolumn{1}{|c|}{$30$} & $852n$ & $2116n$ \\ 
\hline \end{tabular}
\quad
\begin{tabular}{C{15pt}|c|c|}
\cline{2-3}
 & \multicolumn{2}{c|}{$\varepsilon = 10^{-9}, \kappa = 10^9$} \\ \hline
\multicolumn{1}{|c|}{$\kk$} & $\mathcal{W}_\mu(\kk)$ & $\mathcal{W}_0(\kk)$ \\ \hline \hline
\multicolumn{1}{|c|}{$15$} & $1251n$ & $4834n$ \\ 
\multicolumn{1}{|c|}{$24$} & $\textbf{1076n}$ & $3189n$ \\ 
\multicolumn{1}{|c|}{$30$} & $1102n$ & $3018n$ \\ 
\multicolumn{1}{|c|}{$32$} & $1119n$ & $\textbf{3008n}$ \\ 
\multicolumn{1}{|c|}{$40$} & $1210n$ & $3078n$ \\ 
\hline \end{tabular}
\quad
\newline
\end{center}
\caption{Comparison of work sufficient to solve a problem with $n = 10^9$, and various values of $\kappa$ and $\varepsilon$. The work was computed using formula \eqref{eq:shd6dhd7}, with $m(\kk)$ as in \eqref{eq:m:nuismu0_2}. The notation ${\cal W}_\nu(\kk)$ indicates what parameter $\nu$ was used.}
\label{tbl:suboptbounds}
\end{table}

\section{Complexity Analysis: Convex Loss} \label{SEC:convex}

If $f$ is convex but not strongly convex, we define  $\hat{f}_i(\xx)\eqdef f_i(\xx)+\tfrac{\mu}{2}\|\xx-\xx_0\|^2$, for small enough $\mu>0$ (we shall see below how the choice of $\mu$ affects the results), and consider the perturbed problem
\begin{equation}\label{eq:shs7hss}\min_{x \in \R^d} \hat{f}(x),\end{equation} 
where 
\begin{equation}\label{eq:barf} \hat{f}(x) \eqdef \frac{1}{n} \sum_{i=1}^n\hat{f}_i(x) = f(x) + \frac{\mu}{2}\|x-x_0\|^2.\end{equation}
Note that 
 $\hat{f}$ is $\mu$-strongly convex and  $(L+\mu)$-smooth. In particular, the theory developed in the previous section applies. 
We propose that S2GD be instead applied to the perturbed problem, and show that an approximate solution of \eqref{eq:shs7hss} is also an approximate solution of \eqref{eq:main} (we will assume that this problem has a minimizer).  

Let $\hat{\xx}_*$ be the (necessarily unique) solution of the perturbed problem \eqref{eq:shs7hss}. The following result describes an important connection between the original problem and the perturbed problem.
 
\begin{lemma} \label{eq:lemma87878}If $\hat{\xx}\in \R^d$ satisfies $\hat{f}(\hat{\xx})\leq \hat{f}(\hat{x}_*) + \delta$, where $\delta>0$, then \[f(\hat{\xx})\leq f(\xx_*) +  \frac{\mu}{2}\|\xx_0-\xx_*\|^2 + \delta.\]
\end{lemma}
\begin{proof} The statement is almost identical to  Lemma~9 in \cite{richtarik}; its proof follows the same steps with only minor adjustments.
\end{proof}

We are now ready to establish a complexity result for non-strongly convex losses.

\begin{theorem} Let Assumption~\ref{ass:Lip} be satisfied. Choose $\mu>0$, $0\leq \nu \leq \mu$, stepsize $0< h  < \tfrac{1}{2(L+\mu)}$, and let $m$ be sufficiently large so that \begin{equation}\label{eq:hshshs7} \hat{\cc} \eqdef \frac{(1 -\nu h)^m}{\beta \mu h (1 - 2(L+\mu)h)} + \frac{2 Lh}{1 - 2(L+\mu)h} < 1. \end{equation}
Pick $x_0\in \R^d$ and let $\hat{x}_0 = x_0, \hat{x}_1,\dots,\hat{x}_{\kk}$ be the sequence of iterates produced by S2GD as applied to problem \eqref{eq:shs7hss}.
Then, for any $ 0 < \rho < 1 $, $ 0 < \varepsilon < 1 $ and  \begin{equation}
\kk \geq \frac{\log \left(1/(\varepsilon \rho)\right)}{\log (1/\hat{\cc})}, 
\label{eq:hprobs2}
\end{equation} 
we have \begin{equation}\label{eq:prob2} \Prob\left(f(\hat{\xx}_{\kk}) - f(\xx_*) \leq \varepsilon(f(\xx_0)-f(\xx_*)) + \frac{\mu}{2}\|x_0-x_*\|^2\right) \geq 1 - \rho. \end{equation}
In particular, if we choose $\mu=\epsilon< L$ and parameters $\kk^*$, $h(\kk^*)$, $m(\kk^*)$ as in Theorem~\ref{thm:main2}, the amount of work  performed by S2GD to guarantee \eqref{eq:prob2} is
\[{\cal W}(\kk^*,h(\kk^*),m(\kk^*)) = O\left((n+\tfrac{L}{\varepsilon})\log(\tfrac{1}{\varepsilon})\right),\]
which consists of $O(\tfrac{1}{\varepsilon})$ full gradient evaluations and $O(\tfrac{L}{\epsilon}\log(\tfrac{1}{\varepsilon}))$ stochastic gradient evaluations.

\label{thm:hpresult2}
\end{theorem}

\begin{proof}
We first note that
\begin{equation}\label{eq:uddjd8}\hat{f}(\hat{\xx}_0)-\hat{f}(\hat{\xx}_*) \overset{\eqref{eq:barf}}{=} f(\hat{\xx}_0)-\hat{f}(\hat{\xx}_*) \leq f(\hat{x}_0) - f(\hat{x}_*) \leq f(x_0)-f(x_*),\end{equation}
where the first inequality follows from $f\leq \hat{f}$, and the second one from optimality of $x_*$. Hence, by first applying
Lemma~\ref{eq:lemma87878} with $\hat{x}=\hat{x}_{\kk}$ and $\delta= \varepsilon(f(\xx_0)-f(\xx_*))$, and then
 Theorem~\ref{thm:hpresult}, with $c\leftarrow \hat{c}$, $f\leftarrow \hat{f}$, $x_0\leftarrow \hat{x}_0$, $x_*\leftarrow \hat{x}_* $, we obtain 
\begin{eqnarray*}
\Prob\left(f(\hat{\xx}_{\kk}) - f(\xx_*) 
\leq  \delta +\frac{\mu}{2}\|\xx_0-\xx_*\|^2\right) &\overset{(\text{Lemma~}\ref{eq:lemma87878})}{\geq}&
\Prob\left(\hat{f}(\hat{x}_{\kk})-\hat{f}(\hat{x}_*)\leq \delta \right) \\
& \overset{\eqref{eq:uddjd8}}{\geq} & \Prob\left(
\frac{\hat{f}(\hat{x}_{\kk})-\hat{f}(\hat{x}_*)}{\hat{f}(\hat{x}_0)-\hat{f}(\hat{x}_*)} \leq \varepsilon \right) \;\; \overset{\eqref{eq:sjnd8djd}}{\geq} \;\;1-\rho.\end{eqnarray*}
The second statement follows directly from the second part of Theorem~\ref{thm:main2} and the fact that the condition number of the perturbed problem is $\kappa = \tfrac{L+\epsilon}{\epsilon} \leq \tfrac{2L}{\epsilon}$. 
\end{proof}

\section{Numerical Experiments}
\label{SEC:NUMERICS}

In this section we conduct computational experiments to illustrate some aspects of the performance of our algorithm. In Section~\ref{sec:theoryvspractice} we consider the least squares problem with synthetic data to compare the practical performance and the theoretical bound on convergence in expectations. We demonstrate that for both SVRG and S2GD, the practical rate is substantially better than the theoretical one. In Section~\ref{sec:sparses2gd} we explain an efficient way to implement the S2GD algorithm for sparse datasets. In Section~\ref{sec:othermethods} we compare the S2GD algorithm on several real datasets with other algorithms suitable for this task. We also provide efficient implementation of the algorithm for the case of logistic regression in the MLOSS repository\footnote{\url{http://mloss.org/software/view/556/}}.

\subsection{Comparison with theory} 
\label{sec:theoryvspractice}

Figure~\ref{fig:ThVsPracLin} presents a comparison of the theoretical rate and practical performance on a larger problem with artificial data, with a condition number we can control (and choose it to be poor). In particular, we consider the L2-regularized least squares  with $$ f_i(x) = \frac{1}{2}(a_i^Tx - b_i)^2 + \frac{\lambda}{2} \|x\|^2, $$ for some $a_i \in \R^d$, $b_i \in \R$ and $\lambda>0$ is the regularization parameter.

\begin{figure}[!h]
\begin{center}
\includegraphics[width = 5.5in]{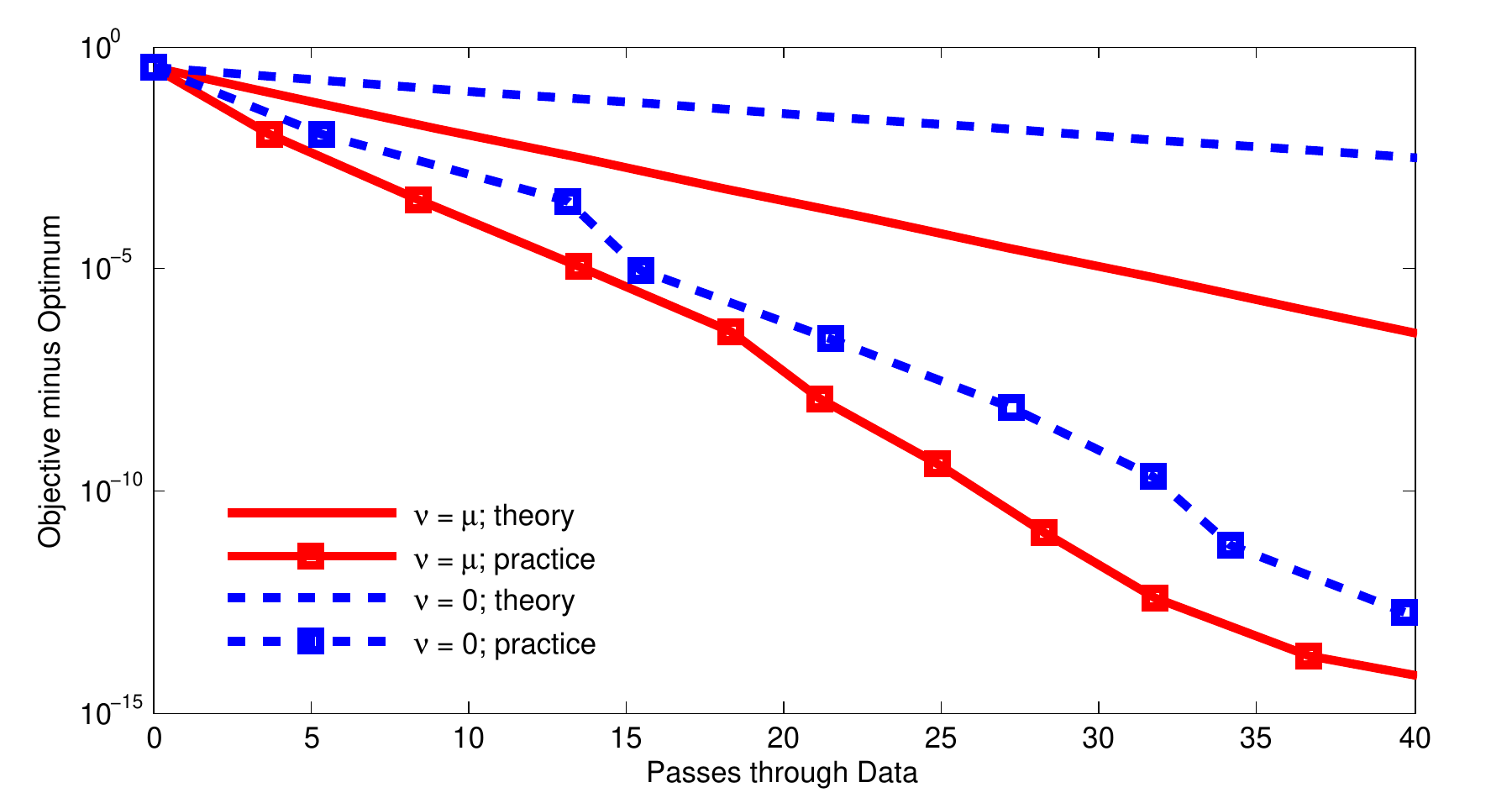}
\end{center}
\caption{Least squares with $n = 10^5$, $\kappa = 10^4$. Comparison of theoretical result and practical performance for cases $\nu = \mu$ (full red line) and $\nu = 0$ (dashed blue line).}
\label{fig:ThVsPracLin}
\end{figure}

We consider an instance with $n = 100,000$, $d = 1,000$ and $\kappa = 10,000.$ We run the algorithm with both parameters $\nu = \lambda$ (our best estimate of $\mu$) and $\nu = 0$. Recall that the latter choice leads to the SVRG method of Johnson and Zhang~\cite{svrg}. We chose parameters $m$ and $h$ as a (numerical) solution of the work-minimization problem \eqref{eq:pracReq}, obtaining $m = 261,063$ and $h = 1/11.4L$ for $\nu = \lambda$ and $m = 426,660$ and $h = 1/12.7L$ for $\nu = 0$. The practical performance is obtained after a single run of the S2GD algorithm. 

The figure demonstrates linear convergence of S2GD in practice, with the convergence rate being significantly better than the already strong theoretical result. Recall that the bound is on the expected function values. We can observe a rather strong convergence to machine precision in work equivalent to evaluating the full gradient only $40$ times. Needless to say, neither  SGD nor GD have such speed.  Our method is also an improvement over \cite{svrg}, both in theory and practice.

\subsection{Implementation for sparse data}
\label{sec:sparses2gd}

In our sparse implementation of Algorithm ~\ref{SVRG}, described in this section and formally stated  as Algorithm~\ref{alg:SVRGsparse},  we make the following structural assumption:

\begin{assumption}\label{ass:sparse} The loss functions arise as the composition of a univariate smooth loss function $\phi_i$, and an inner product with a data point/example $a_i\in \R^d$: \[f_i(x) = \phi_i(a_i^T x), \qquad i=1,2,\dots,n.\]  In this case, $ f'_i(x) =  \phi_i'(a_i^T x) a_i$.
\end{assumption}

This is the  structure   in many cases of interest, including linear or logistic regression.  

A natural question one might want to ask is whether S2GD can be implemented efficiently for sparse data. 

Let us first take a brief detour and look at SGD, which performs iterations of the type:
\begin{equation}\label{eq:sgdxx}x_{\kk+1} \leftarrow x_\kk - h \phi_i'(a_i^T x) a_i.\end{equation}
 Let  $\omega_i$ be the number of nonzero features in example $a_i$, i.e., $\omega_i \eqdef \|a_i\|_0 \leq d$. Assuming that the computation of the derivative of the univariate function $\phi_i$ takes $O(1)$ amount of work, 
the computation of $\nabla f_i(x)$ will take $O(\omega_i)$ work. Hence, the update step \eqref{eq:sgdxx}  will  cost $O(\omega_i)$, too, which means the method can naturally speed up its iterations on sparse data.

The situation is not as simple with  S2GD, which for loss functions of the type described in Assumption~\ref{ass:sparse}  performs inner iterations as follows:
\begin{equation} \label{eq:s2gd-update}\yy_{\kk,t+1} \gets \yy_{\kk,t} - h \left( \grad_{\kk} + \phi_i'(a_i^T \yy_{\kk,t} )a_i - \phi_i'( a_i^T \xx_{\kk}) a_i  \right) .\end{equation}
 Indeed, note that $g_j = f'(x_j)$ is in general be fully dense even for sparse data $\{a_i\}$. As a consequence, the update in \eqref{eq:s2gd-update}  might be as costly as $d$ operations, irrespective of the sparsity level $\omega_i$ of the active  example $a_i$.  However, we can use the following ``lazy/delayed'' update trick. 
  We split the update to the $y$ vector into two parts: immediate, and delayed. Assume index $i=i_t$ was chosen at  inner iteration $t$. We immediately perform the update 
 \[\tilde{y}_{j,t+1} \leftarrow y_{j,t} - h \left(  \phi_{i_t}'(a_{i_t}^T \yy_{\kk,t}) - \phi_{i_t}'(a_{i_t}^T \xx_{\kk}) \right)a_{i_t},\]
 which costs $O(a_{i_t})$. Note that we have not computed the  $y_{j,t+1}$. However, we ``know'' that \[y_{j,t+1} = \tilde{y}_{j,t+1}- h g_j,\]
 without having to actually compute the difference. At the next iteration,  we are supposed to perform update \eqref{eq:s2gd-update} for $i=i_{t+1}$:
\begin{equation} \label{eq:s2gd-updateXX}\yy_{\kk,t+2} \gets \yy_{\kk,t+1} - h  \grad_{\kk} - h \left(\phi_{i_{t+1}}'(a_{i_{t+1}}^T\yy_{\kk,t+1})- \phi_{i_{t+1}}'(a_{i_{t+1}}^T \xx_{\kk})  \right)a_{i_{t+1}} .\end{equation}
 

\begin{algorithm}
\begin{algorithmic}
\State \textbf{parameters:} $m$ = max \# of stochastic steps per epoch, $h$ = stepsize, $\nu$ = lower bound on $\mu$
\For {$\kk = 0, 1, 2, \dots$}
	\State $\grad_{\kk} \gets \frac{1}{n} \sum_{i=1}^n f_i'(\xx_{\kk})$
	\State $\yy_{\kk,0} \gets \xx_{\kk}$
	\State $\chi^{(s)} \gets 0$ for $s = 1, 2, \dots, d$	
	\Comment Store when a coordinate was updated last time
	\State Let $t_{\kk} \gets t$ with probability $(1 - \nu h)^{m-t} / \beta $ for $t = 1, 2, \dots, m$
	\For {$t = 0$ to $t_{\kk}-1$}
		\State Pick $i \in \{ 1, 2, \dots, n \}$, uniformly at random
		\For {$s \in \text{nnz}(a_i)$}
			\State $\yy_{\kk,t}^{(s)} \gets \yy_{\kk,t}^{(s)} - (t - \chi^{(s)}) h g_j^{(s)} $
			\Comment Update what will be needed
			\State $\chi^{(s)} = t$ 
		\EndFor
		\State $ \yy_{\kk,t+1} \gets \yy_{\kk,t} - h \left( \phi_i'(\yy_{\kk,t}) - \phi_i'(\xx_{\kk})  \right)a_i $
		\Comment A sparse update
	\EndFor
	\For {$ s = 1$ to $d$} \Comment Finish all the ``lazy'' updates 
		\State $\yy_{\kk, t_{\kk}}^{(s)} \gets \yy_{\kk, t_{\kk}}^{(s)} - (t_j - \chi^{(s)}) h   g_j^{(s)} $
	\EndFor
	\State $\xx_{\kk+1} \gets \yy_{\kk, t_{\kk}}$
\EndFor
\end{algorithmic}

\caption{Semi-Stochastic Gradient Descent (S2GD) for sparse data, using ``lazy'' updates}
\label{alg:SVRGsparse}
\end{algorithm}

However, notice that we can't compute 
\begin{equation}\label{eq:iusi89s}\phi_{i_{t+1}}'(a_{i_{t+1} }^Ty_{j,t+1})\end{equation}
 as we never computed $y_{j,t+1}$. However, here lies the trick: as $a_{i_{t+1}}$ is sparse, we only need to know those coordinates $s$ of $y_{j,t+1}$ for which $a_{i_{t+1}}^{(s)}$ is nonzero. So, just before we compute the (sparse part of) of the update \eqref{eq:s2gd-updateXX}, we perform the update 
\[y_{j,t+1}^{(s)} \leftarrow \tilde{y}_{j,t+1}^{(s)} - h g_j^{(s)}\]
for coordinates $s$ for which  $a_{i_{t+1}}^{(s)}$ is nonzero. This way we know that the inner product appearing in \eqref{eq:iusi89s} is computed correctly (despite the fact that $y_{j,t+1}$ potentially is not!). In turn, this means that we can compute the sparse part of the update in \eqref{eq:s2gd-updateXX}. 

We now continue as before, again only computing $\tilde{y}_{j,t+3}$. However, this time we have to be more careful as  it is no longer true that
\[y_{j,t+2} = \tilde{y}_{j,t+2} - h g_j.\]
We need to remember, for each coordinate $s$, the last iteration counter $t$ for which $a_{i_t}^{(s)}\neq 0$. This way we will know how many times did we ``forget'' to apply the dense update $-h g_j^{(s)}$. We do it in a just-in-time fashion, just before it is needed.

Algorithm~\ref{alg:SVRGsparse} (sparse S2GD) performs these lazy updates as described above. It produces exactly the same result as Algorithm~\ref{SVRG} (S2GD), but is much more efficient for sparse data as  iteration picking example $i$ only costs $O(\omega_i)$. This is done with a memory overhead of only $O(d)$ (as represented by vector $\chi \in \R^d$).



\subsection{Comparison with other methods}
\label{sec:othermethods}

The S2GD algorithm can be applied to several classes of problems. We perform experiments on an important and in many applications used L2-regularized logistic regression for binary classification on several datasets. The functions $f_i$ in this case are: $$ f_i(x) = \log\left(1 + \exp\left(l_i a_i^T x \right) \right) + \frac{\lambda}{2} \| x \|^2, $$ where $l_i$ is the label of $i^{th}$ training exapmle $a_i$. In our experiments we set the regularization parameter $\lambda = O(1/n)$ so that the condition number $\kappa = O(n)$, which is about the most ill-conditioned problem used in practice. We added a (regularized) bias term to all datasets.

All the datasets we used, listed in Table~\ref{tbl:datasets}, are freely available\footnote{Available at \href{http://www.csie.ntu.edu.tw/~cjlin/libsvmtools/datasets/}{http://www.csie.ntu.edu.tw/$\sim$cjlin/libsvmtools/datasets/}.} benchmark binary classification datasets.

\begin{table}[!h]
\centering
\begin{tabular}{| r | r | r | r | r | r |} 
\hline
Dataset & Training examples ($n$) & Variables ($d$) & $L$ & $\mu$ & $\kappa$ \\
\hline
\textit{ijcnn} & 49 990 & 23 & 1.23 & 1/$n$ & 61 696 \\
\textit{rcv1} & 20 242 & 47 237 & 0.50 & 1/$n$ & 10 122 \\
\textit{real-sim} & 72 309 & 20 959 & 0.50 & 1/$n$ & 36 155 \\
\textit{url} & 2 396 130 & 3 231 962 & 128.70 & 100/$n$ & 3 084 052 \\
\hline 
\end{tabular}
\caption{Datasets used in the experiments.}
\label{tbl:datasets}
\end{table}

In the experiment, we compared the following algorithms:
\begin{itemize}
\item \textbf{SGD:} Stochastic Gradient Descent. After various experiments, we decided to use a variant with constant step-size that gave the best practical performance in hindsight.
\item \textbf{L-BFGS:} A publicly-available limited-memory quasi-Newton method that is suitable for broader classes of problems. We used a popular implementation by Mark Schmidt.\footnote{\href{http://www.di.ens.fr/~mschmidt/Software/minFunc.html}{http://www.di.ens.fr/$\sim$mschmidt/Software/minFunc.html}}
\item \textbf{SAG:} Stochastic Average Gradient \cite{SAGjournal2013}. This is the most important method to compare to, as it also achieves linear convergence using only stochastic gradient evaluations. Although the methods has been analysed for stepsize $h = 1/16L$, we experimented with various stepsizes and chose the one that gave the best performance for each problem individually.
\item \textbf{S2GDcon:} The S2GD algorithm with conservative stepsize choice, i.e., following the theory. We set $m = O(\kappa)$ and $h = 1/10L$, which is approximately the value you would get from Equation~\eqref{eq:shhsdd998}
\item \textbf{S2GD:} The S2GD algorithm, with stepsize that gave the best performance in hindsight.
\end{itemize}

Note that SAG needs to store $n$ gradients in memory in order to run. In case of relatively simple functions, one can store only $n$ scalars, as the gradient of $f_i$ is always a multiple of $a_i$. If we are comparing with SAG, we are implicitly assuming that our memory limitations allow us to do so. Although not included in Algorithm~\ref{SVRG}, we could also store these gradients we used to compute the full gradient, which would mean we would only have  to compute a single stochastic  gradient per inner iteration (instead of two).

We plot the results of these methods, as applied to various different, in the Figure~\ref{fig:plotTogether} for first  15-30 passes through the data (i.e., amount of work work equivalent to 15-30 full gradient evaluations).

\begin{figure}[!h]
\begin{center}
\includegraphics[angle = 270, width = \linewidth]{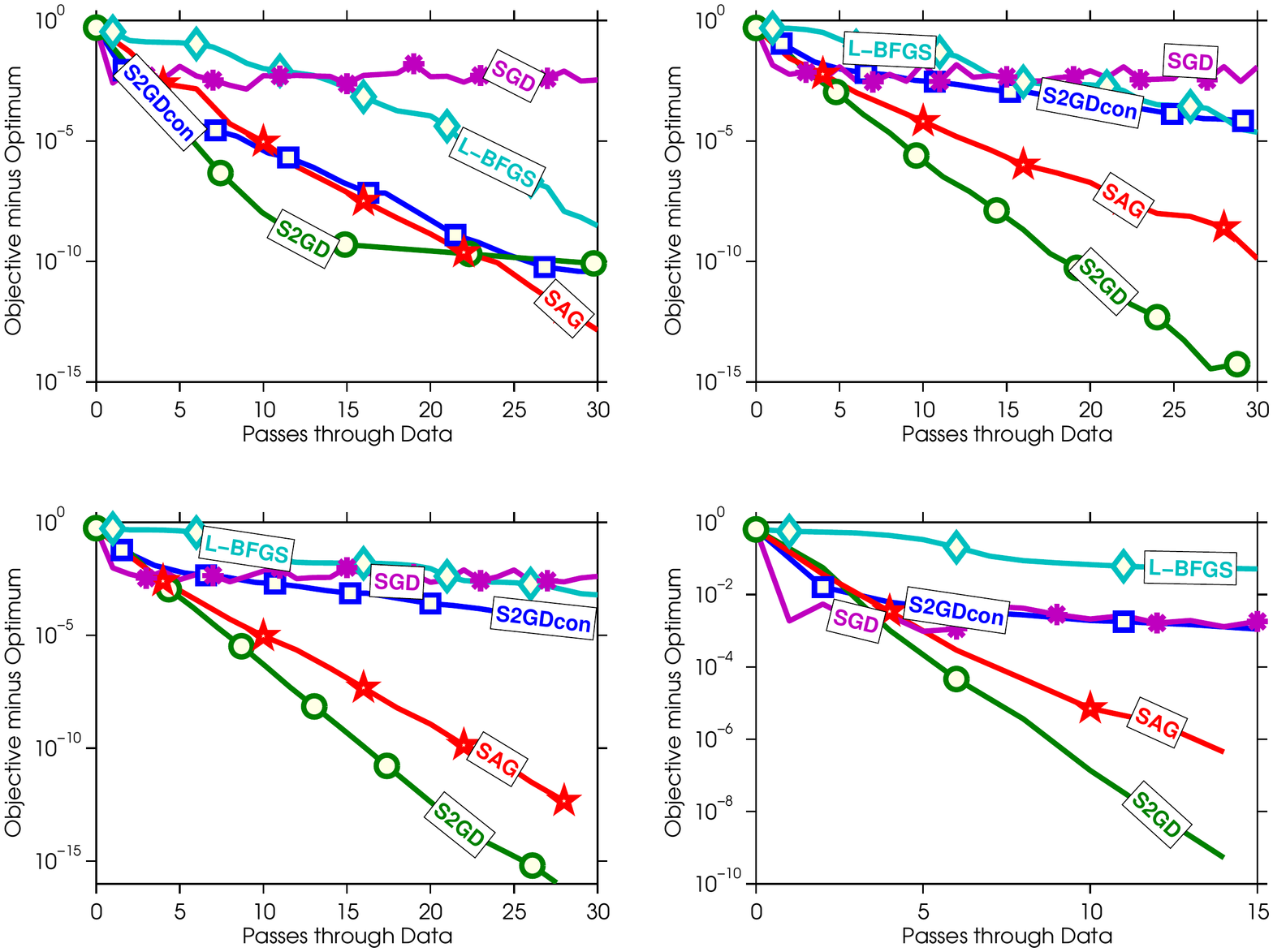}
\end{center}
\caption{Practical performance for logistic regression and  the following datasets: \textit{ijcnn, rcv} (first row), \textit{realsim, url} (second row)}
\label{fig:plotTogether}
\end{figure}

There are several remarks  we would like to make. First, our experiments confirm the insight from \cite{SAGjournal2013} that for this types of problems, reduced-variance methods consistently exhibit substantially better performance than the popular L-BFGS algorithm.

The performance gap between S2GDcon and S2GD differs from dataset to dataset. A possible explanation for this can be found in an extension of SVRG to proximal setting \cite{proxSVRG}, released after the first version of this paper was put onto arXiv (i.e., after December 2013) . Instead Assumption~\ref{ass:Lip}, where all loss functions are assumed to be associated with  the same constant $L$, the authors of \cite{proxSVRG} instead assume that each loss function $f_i$ has its own constant  $L_i$. Subsequently, they sample proportionally to these quantities as opposed to the  uniform sampling. In our case, $L = \max_i L_i$. This importance sampling has an impact on the convergence: one gets dependence on the average of the quantities $L_i$ and not in their maximum.


The number of passes through data seems a reasonable way to compare performance, but some algorithms could need more time to do the same amount of passes through data than others. In this sense, S2GD should be in fact faster than SAG due to the following property. While SAG updates the test point after each evaluation of a stochastic gradient, S2GD does not always make the update --- during the evaluation of the full gradient. This claim is supported by computational evidence: SAG needed about 10-30\% more time than S2GD to do the same amount of passes through data.

Finally, in Table~\ref{tbl:experimentsTime} we provide the time it took the algorithm  to produce these plots on a desktop computer with Intel Core i7 3610QM processor, with 2 $\times$ 4GB DDR3 1600 MHz memory. The number for L-BFGS at the \textit{url} dataset is not representative, as the algorithm needed extra memory, which slightly exceeded the memory limit of our computer.

\begin{table}
\centering
\begin{tabular}{c|r|r|r|r|}
\cline{2-5}
 & \multicolumn{4}{c|}{Time in seconds} \\
\hline
\multicolumn{1}{|c|}{Algorithm} & \textit{ijcnn} & \textit{rcv1} & \textit{real-sim} & \textit{url} \\
\hline
\multicolumn{1}{|c|}{S2GDcon} & 0.42 & 0.89 & 1.74 & 78.13 \\
\multicolumn{1}{|c|}{S2GD}    & 0.51 & 0.98 & 1.65 & 85.17 \\
\multicolumn{1}{|c|}{SAG}     & 0.64 & 1.23 & 2.64 & 98.73 \\
\multicolumn{1}{|c|}{L-BFGS}  & 0.37 & 1.08 & 1.78 & 830.65 \\
\hline
\end{tabular}
\caption{Time required to produce plots in Figure~\ref{fig:plotTogether}.}
\label{tbl:experimentsTime}
\end{table}

\subsection{Boosted variants of S2GD and SAG}
\label{sec:plusVariants}

In this section we study the practical performance of boosted methods, namely S2GD+ (Algorithm~\ref{alg:S2GD+}) and variant of SAG suggested by its authors \cite[Section 4.2]{SAGjournal2013}.

\begin{figure}[!h]
\begin{center}
\includegraphics[angle = 270, width = \linewidth]{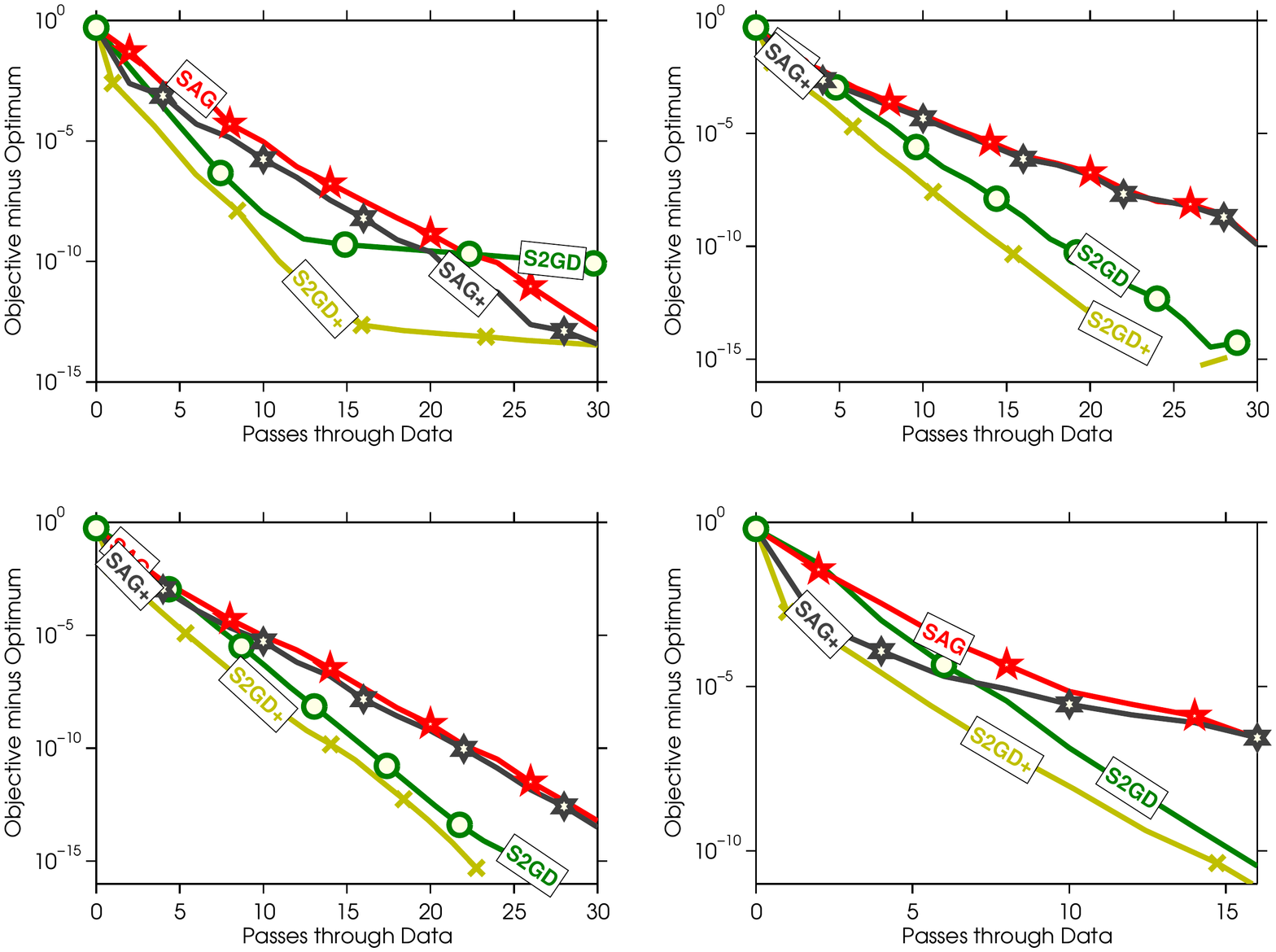}
\end{center}
\caption{Practical performance of boosted methods on datasets \textit{ijcnn, rcv} (first row), \textit{realsim, url} (second row)}
\label{fig:plotTogetherPlus}
\end{figure}

 SAG+  is a simple modification of SAG, where one does not divide the  sum of the stochastic gradients by $n$, but by the number of training examples seen during the run of the algorithm, which has the effect of producing larger steps at the beginning. The authors claim that this method performed better in practice than a hybrid SG/SAG algorithm. 
 
 
We have observed that, in practice, starting SAG from a point close to the optimum, leads to an initial  ``away jump``. Eventually, the method exhibits   linear convergence. In contrast, S2GD converges linearly from the start, regardless of the starting position.

Figure~\ref{fig:plotTogetherPlus} shows that S2GD+ consistently improves over S2GD, while SAG+ does not improve always: sometimes it performs essentially the same as SAG. Although S2GD+ is overall a superior algorithm, one should note that this comes at the cost of having to choose stepsize parameter for SGD initialization. If one chooses these parameters poorly, then S2GD+ could perform worse than S2GD. The other three algorithms can work well without any parameter tuning.

\section{Conclusion} \label{SEC: CONCLUDE}
We have developed a new semi-stochastic gradient descent method (S2GD) and analyzed its complexity for smooth convex and strongly convex loss functions. Our methods  need $O((\kappa/n)\log(1/\varepsilon))$ work only, measured in  units equivalent to the evaluation of the full gradient of the loss function, where $\kappa=L/\mu$  if the loss is $L$-smooth and $\mu$-strongly convex, and $\kappa\leq 2L/\varepsilon$ if the loss is merely $L$-smooth.

Our results in the strongly convex case match or improve on a few very recent results, while at the same time generalizing and simplifying the analysis. Additionally, we proposed  S2GD+ ---a method which equips S2GD with an SGD pre-processing step---which in our experiments exhibits superior performance to all methods we tested. We leave the analysis of this method as an open problem. 


\bibliography{notes}

\end{document}